\documentclass[11pt]{article}
\usepackage{natbib}
\usepackage[utf8]{inputenc}
\usepackage{amssymb,amsmath,amsthm,amsfonts,dsfont}
\usepackage{bbm}
\usepackage[algo2e, ruled,vlined,linesnumbered]{algorithm2e}
\usepackage{tcolorbox}
\usepackage{upgreek}
\usepackage{graphicx}

\usepackage{fullpage}

\newcommand\cA{\mathcal{A}}

\newcommand\cP{\mathcal{P}}
\newcommand\cQ{\mathcal{Q}}

\newcommand\cX{\mathcal{X}}
\newcommand\cD{\mathcal{D}}

\newcommand\E{\mathop{\mathbb{E}}}

\newcommand{\1}{\mathds{1}}
\newcommand{\boldx}{\boldsymbol{x}}
\newcommand{\pirv}{\Pi}
\newcommand{\pispace}{\boldsymbol{\Pi}}
\DeclareMathOperator*{\argmax}{arg\,max}

\newtheorem{theorem}{Theorem}[section]

\newtheorem{lemma}[theorem]{Lemma}

\newtheorem*{remark}{Remark}
\newtheorem{definition}{Definition}[section]

\newif\ifcomment
\commenttrue

\ifcomment
\newcommand{\ar}[1]{\textcolor{blue}{[AR: #1]}}

\else
\newcommand{\ar}[1]{}
\fi

\title{A New Analysis of Differential Privacy's Generalization Guarantees}
\author{Christopher Jung\thanks{Supported in part by NSF grant AF-1763307} \and Katrina Ligett\thanks{Supported in part by the United States Air Force and DARPA under Contract No FA8750-16-C-0022.} \and Seth Neel\thanks{Supported in part by an NSF Graduate Research Fellowship} \and Aaron Roth\thanks{Supported in part by NSF grant AF-1763314, the United States Air Force and DARPA under Contract No FA8750-16-C-0022, and a grant from the Sloan Foundation.} \and Saeed Sharifi-Malvajerdi \and Moshe Shenfeld}

\begin{document}

\maketitle
\begin{abstract}
We give a new proof of the ``transfer theorem'' underlying adaptive data analysis: that any mechanism for answering adaptively chosen statistical queries that is differentially private and sample-accurate is also accurate out-of-sample. Our new proof is elementary and gives structural insights that we expect will be useful elsewhere. We show: 1) that differential privacy ensures that the expectation of any query on the \emph{posterior distribution} on datasets induced by the transcript of the interaction is close to its true value on the data distribution, and 2) sample accuracy on its own ensures that any query answer produced by the mechanism is close to its posterior expectation with high probability. This second claim follows from a thought experiment in which we imagine that the dataset is resampled from the posterior distribution after the mechanism has committed to its answers. The transfer theorem then follows by summing these two bounds, and in particular, avoids the ``monitor argument'' used to derive high probability bounds in prior work.

An upshot of our new proof technique is that the concrete bounds we obtain are substantially better than the best previously known bounds, even though the improvements are in the constants, rather than the asymptotics (which are known to be tight). As we show, our new bounds outperform the naive ``sample-splitting'' baseline at dramatically smaller dataset sizes compared to the previous state of the art, bringing techniques from this literature closer to practicality.
\end{abstract}
\thispagestyle{empty} \setcounter{page}{0}
\clearpage

\section{Introduction}
Many data analysis pipelines are \emph{adaptive}: the choice of which analysis to run next depends on the outcome of previous analyses. Common examples include variable selection for regression problems and hyper-parameter optimization in large-scale machine learning problems: in both cases, common practice involves repeatedly evaluating a series of models on the same dataset. Unfortunately, this kind of adaptive re-use of data invalidates many traditional methods of avoiding over-fitting and false discovery, and has been blamed in part for the recent flood of non-reproducible findings in the empirical sciences \citep{GL14}.

There is a simple way around this problem: don't re-use data. This idea suggests a baseline called \emph{data splitting}: to perform $k$ analyses on a dataset, randomly partition the dataset into $k$ disjoint parts, and perform each analysis on a fresh part. The standard ``holdout method'' is the special case of $k=2$. Unfortunately, this natural baseline makes poor use of data: in particular, the data requirements of this method grow \emph{linearly} with the number of analyses $k$ to be performed.

A recent literature starting with \citet{DFHPRR15a} shows how to give a significant asymptotic improvement over this baseline via a connection to differential privacy: rather than computing and reporting exact sample quantities, perturb these quantities with noise. This line of work established a powerful \emph{transfer theorem}, that informally says that any analysis that is simultaneously differentially private and accurate \emph{in-sample} will also be accurate \emph{out-of-sample}. The best analysis of this technique shows that for a broad class of analyses and a target accuracy goal, the data requirements grow only with $\sqrt{k}$ --- a quadratic improvement over the baseline \citep{BNSSSU16}. Moreover, it is known that in the worst case, this cannot be improved asymptotically \citep{HU14,SU15}. Unfortunately, thus far this literature has had little impact on practice. One major reason for this is that although the more sophisticated techniques from this literature give asymptotic improvements over the sample-splitting baseline, the concrete bounds do not actually improve on the baseline until the dataset is enormous. This remains true even after optimizing the constants that arise from the arguments of \citep{DFHPRR15a} or \citep{BNSSSU16}, and appears to be a fundamental limitation of their proof techniques \citep{RRSSTW19}. In this paper, we give a new proof of the transfer theorem connecting differential privacy and in-sample accuracy to out-of-sample accuracy. Our proof is based on a simple insight that arises from taking a Bayesian perspective, and in particular yields an improved concrete bound that beats the sample-splitting baseline at dramatically smaller data set sizes $n$ compared to prior work. In fact, at reasonable dataset sizes, the magnitude of the improvement arising from our new theorem is significantly larger than the improvement between the bounds of \citet{BNSSSU16} and \citet{DFHPRR15a}: see Figure \ref{fig:intro}.

\begin{figure}[h]
\label{fig:intro}
\begin{center}
\includegraphics[scale=0.2]{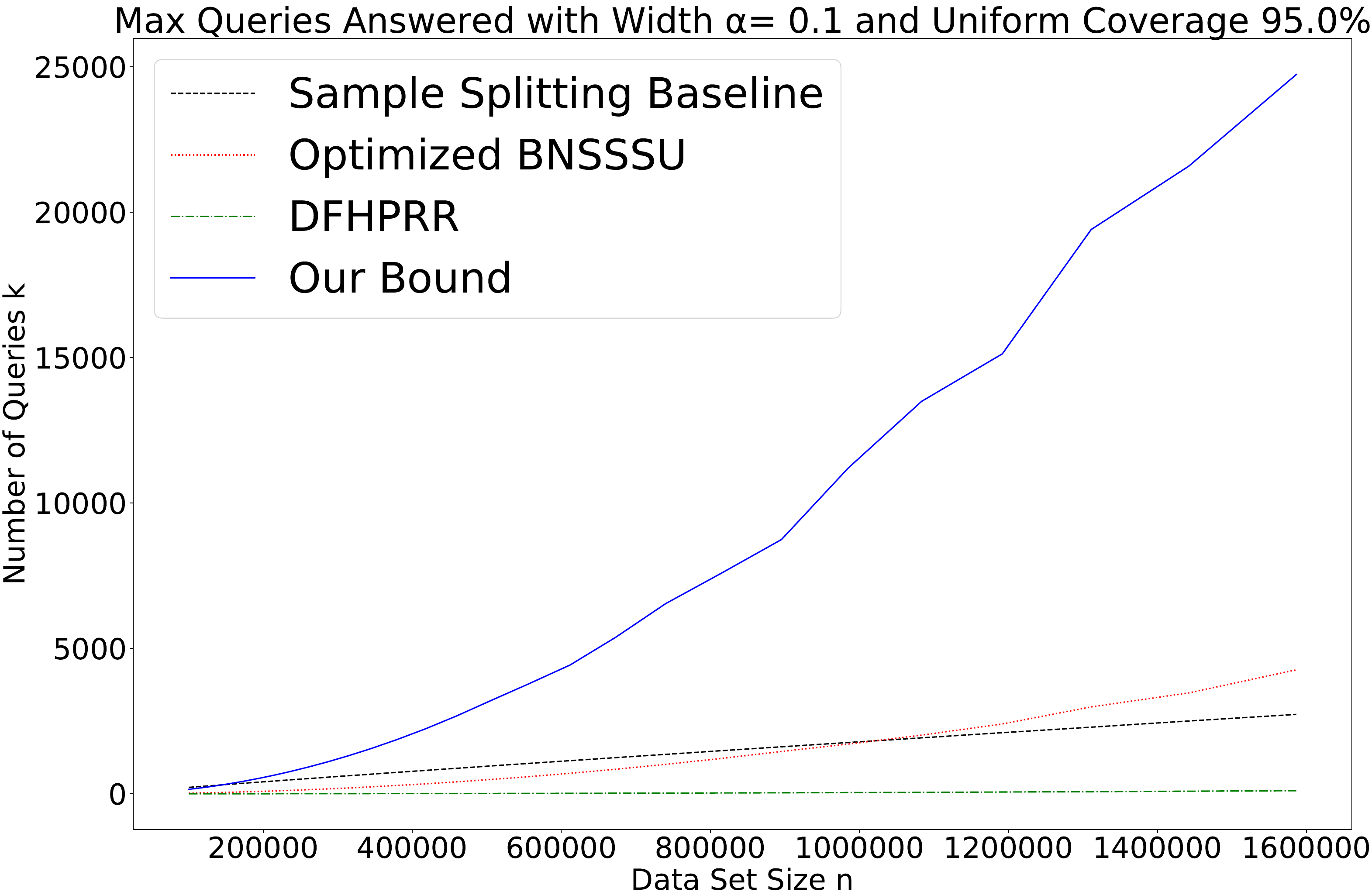}
\end{center}
\caption{A comparison of the number of adaptive linear queries that can be answered using the Gaussian mechanism as analyzed by our transfer theorem (Theorem \ref{thm:transfer2}), the numerically optimized variant of the bound from \citet{BNSSSU16} as derived in \cite{RRSSTW19}, and the original transfer theorem from \cite{DFHPRR15a}. We plot for each dataset size $n$, the number of queries $k$ that can be answered while guaranteeing confidence intervals around the answer that have width $\alpha = 0.1$ and uniform coverage probability $1-\beta = 0.95$. We compare with the naive sample splitting baseline that simply splits the dataset into $k$ pieces and answers each query with the empirical answer on a fresh piece.}
\end{figure}

\subsection{Proof Techniques}
\paragraph{Prior Work}
Consider an unknown data distribution $\cP$ over a data-domain $\cX$, and a dataset $S \sim \cP^n$ consisting of $n$ i.i.d. draws from $\cP$. It is a folklore observation (attributed to Frank McSherry) that if a predicate $q:\cX\rightarrow [0,1]$ is selected by an $\epsilon$-differentially private algorithm $M$ acting on $S$, then it will generalize \emph{in expectation} (or \emph{have low bias}) in the sense that $|\E_{q\sim M(S)}[\E_{x \sim \cP}[q(x)] - \frac{1}{n}\sum_{x\in S}q(x)]|  \approx \epsilon$.
But bounds on bias are not enough to yield confidence intervals (except through Markov's inequality), and so prior work has focused on strengthening the above observation into a high probability bound. For small $\epsilon$, the optimal bound has the asymptotic form: $\Pr_{q\sim M(S)}[|\E_{x \sim \cP}[q(x)] - \frac{1}{n}\sum_{x\in S}q(x)|\geq \epsilon] \leq e^{-O(\epsilon^2n)}$ \citep{BNSSSU16}. Note that this bound does not refer to the \emph{estimated answers} supplied to the data analyst: it says only that a differentially private data analyst is unlikely to be able to find a query whose average value on the dataset differs substantially from its expectation. Pairing this with a simultaneous high probability bound on the \emph{in-sample accuracy} of a mechanism---that it supplies answers $a$ such that with high probability the empirical error is small: $\Pr_{a \sim M(S)}[|a -   \frac{1}{n}\sum_{x\in S}q(x)|\geq \alpha]\leq \beta$---yields a bound on out-of-sample accuracy via the triangle inequality.

 \citet{DFHPRR15a} proved their high probability bound via a direct computation on the \emph{moments} of empirical query values, but this technique was unable to achieve the optimal rate. \citet{BNSSSU16} proved a bound with the optimal rate by introducing the ingenious \emph{monitor technique}. This important technique has subsequently found other uses \citep{SU17,NS19,FV18}, but is a heavy hammer that seems unavoidably to yield large constant overhead, even after numeric optimization \citep{RRSSTW19}.
\paragraph{Our Approach}
We take a fundamentally different approach by directly providing high probability bounds on the out-of-sample accuracy $|a-\E_{x \sim \cP}[q(x)]|$ of mechanisms that are both differentially private and accurate in-sample. Our elementary approach is motivated by the following thought experiment: in actuality, the dataset $S$ is fixed before any interaction with $M$ begins. However, imagine that after the entire interaction with $M$ is complete, the dataset $S$ is \emph{resampled} from the posterior distribution $\cQ$ on datasets \emph{conditioned} on the output of $M$. This thought experiment doesn't alter the joint distribution on datasets and outputs, and so any in-sample accuracy guarantees that $M$ has continue to hold under this hypothetical re-sampling experiment. But because the empirical value of the queries on the re-sampled dataset are likely to be close to their expected value over the posterior $\cQ$, the only way the mechanism can promise to be sample-accurate with high probability is if it provides answers that are \emph{close to their expected value over the posterior distribution with high probability}.

This focuses attention on the posterior distribution on datasets induced by differentially private transcripts. But it is not hard to show that a consequence of differential privacy is that the posterior expectation of any query must be close to its expectation over the data distribution with high probability. In contrast to prior work, this argument directly leverages high-probability in-sample accuracy guarantees of a private mechanism to derive high-probability out-of-sample guarantees, without the need for additional  machinery like the monitor argument of \citep{BNSSSU16}.

\subsection{Further Related Work}
The study of ``adaptive data analysis'' was initiated by \citet{DFHPRR15a,DFHPRR15b} who provided upper bounds via a connection to differential privacy, and \citet{HU14} who provided lower bounds via a connection to fingerprinting codes. The upper bounds were subsequently strengthened by \citet{BNSSSU16}, and the lower bounds by \citet{SU15} to be (essentially) matching, asymptotically. The upper bounds were optimized by \citet{RRSSTW19}, which we use in our comparisons. Subsequent work proved transfer theorems related to other quantities like description length bounds \citep{DFHPRR15c} and compression schemes \citep{CLNRW16},  and expanded the types of analyses whose generalization properties we could reason about via a connection to a quantity called approximate max information \citep{DFHPRR15c,RRST16}. \citet{FS17,FS18} gave improved methods that could guarantee out-of-sample accuracy bounds that depended on query variance. \citet{NR18} extend the transfer theorems from this literature to the related problem of adaptive data gathering, which was identified by \citet{NTTZ18}. \citet{LS19} give an algorithmic stability notion they call \emph{local statistical stability} (also defined with respect to a posterior data distribution) that they show asymptotically characterizes the ability of mechanisms to offer high probability out-of-sample generalization guarantees for linear queries. A related line of work initiated by \citet{RZ16} and extended by \citet{XR17} starts with weaker assumptions on the mechanism (mutual information bounds), and derives weaker conclusions (bounds on bias, rather than high probability generalization guarantees). 

A more recent line of work aims at mitigating the fact that the worst-case bounds deriving from transfer theorems do not give non-trivial guarantees on reasonably sized datasets. \citet{ZH19} show that better bounds can be derived under the assumption that the data analyst is restricted in various ways to not be fully adaptive. \citet{FFH19} showed that overfitting by a classifier because of test-set re-use is mitigated in multi-label prediction problems, compared to binary prediction problems. \citet{RRSSTW19} gave a method for certifying the correctness of heuristically guessed confidence intervals, which they show often out-perform the theoretical guarantees by orders of magnitude.

Finally, \citet{Eld16,Eld16b} proposed a Bayesian reformulation of the adaptive data analysis problem. In the model of \citep{Eld16}, the data distribution $\cP$ is assumed to itself be drawn from a prior that is commonly known to the data analyst and mechanism. In contrast, we work in the standard adversarial setting originally introduced by \citet{DFHPRR15a} in which the mechanism must offer guarantees for worst case data distributions and analysts, and use the Bayesian view purely as a proof technique.  

\section{Preliminaries}
Let $\cX$ be an abstract data domain, and let $\cP$ be an arbitrary distribution over $\cX$. A dataset of size $n$ is a collection of $n$ data records: $S = \{S_i\}_{i=1}^n \in \cX^n$. We study datasets sampled $i.i.d.$ from $\cP$: $S \sim \cP^n$. We will write $S$ to denote the random variable and $\boldx$ for realizations of this random variable. A linear query is a function $q:\cX^* \rightarrow [0,1]$ that takes the following empirical average form when acting on a data set $S \in \cX^n$:
$$
q(S) = \frac{1}{n}\sum_{i=1}^n q(S_i).
$$
We will be interested in estimating the expectations of linear queries over $\cP$. Abusing notation, given a distribution $\cD$ over datasets, we write $q(\cD)$ to denote the expectation of $q$ over datasets drawn from $\cD$, and write $S_i \sim S$ to denote a datapoint sampled uniformly at random from a dataset $S$. Note that for linear queries we have: 
\[  q(\cD) = \E_{S \sim \cD}[q(S)] = \E_{S \sim \cD, S_i \sim S}[q(S_i)]. \]
We note that for linear queries, when the dataset distribution $\cD = \cP^n$, we have $q(\cP^n) = \E_{x \sim \cP}[q(x)]$, which we write as $q(\cP)$ when the notation is clear from context. However, the more general definition will be useful because we will need to evaluate the expectation of $q$ over other (non-product) distributions over datasets in our arguments, and we will generalize beyond linear queries in Appendices \ref{sec:lowsens} and \ref{sec:optquery}.


Given a family of  queries $Q$, a statistical estimator is a (possibly stateful)  randomized algorithm $M:\cX^n\times Q^*\rightarrow \mathbb{R}^*$ parameterized by a dataset $S$ that interactively takes as input a stream of queries $q_i \in Q$, and provides answers $a_i \in \mathbb{R}$. An analyst is an arbitrary randomized algorithm $\cA:\mathbb{R}^*\rightarrow Q^*$ that generates a stream of queries and receives a stream of answers (which can inform the next queries it generates). When an analyst interacts with a statistical estimator, they generate a transcript of their interaction $\pi \in \pispace$ where $\pispace = (Q \times \mathbb{R})^*$ is the space of all transcripts.
 Throughout we write $\pirv$ to denote the transcript's random variable and $\pi$ for its realizations.

  \begin{algorithm2e}
    \caption{Interact$(M, \cA; S)$: An Analyst Interacting with a Statistical Estimator to Generate a Transcript} \label{alg:interact}
    \KwIn{A statistical estimator $M$, an analyst $\cA$, and a dataset $S \in \cX^n$.}
    \For{$t=1$ to $k$}
    {
      The analyst generates a query $q_t \leftarrow \cA(a_1,\ldots,a_{t-1})$ and sends it to the statistical estimator\;
      The statistical estimator generates an answer $a_t \leftarrow M(S; q_t)$\;
    }
    \Return $\pirv = ((q_1,a_1),\ldots,(q_k,a_k))$.
  \end{algorithm2e}
The interaction is summarized in Algorithm \ref{alg:interact}, and we write Interact$(M, \cA; S)$ to refer to it. When $M$ and $\cA$ are clear from context, we will abbreviate this notation and write simply $I(S)$. When we refer to an indexed query $q_j$, this is implicitly a function of the transcript $\pi$. Given a transcript $\pi \in \pispace$, write $\cQ_\pi$ to denote the posterior distribution on datasets conditional on $\pirv = \pi$: $\cQ_\pi = (\cP^n) | \textrm{Interact}(M, \cA; S)= \pi$. Note that $\cQ_{\pi}$ will no longer generally be a product distribution. 
We will be interested in evaluating uniform accuracy bounds, which control the worst-case error over all queries:

\begin{definition}
$M$ satisfies $(\alpha,\beta)$-sample accuracy if for every data analyst $\cA$ and every data distribution $\cP$,
$$\Pr_{S \sim \cP^n, \pirv \sim \textrm{Interact}(M, \cA; S)} [\max_j |q_j(S) - a_j| \geq \alpha] \leq \beta.$$
We say $M$ satisfies $(\alpha,\beta)$-distributional accuracy if for every data analyst $\cA$ and every data distribution $\cP$,
$$\Pr_{S \sim \cP^n, \pirv \sim \textrm{Interact}(M, \cA; S)} [\max_j |q_j(\cP^n) - a_j| \geq \alpha] \leq \beta.$$
\end{definition}


We will be interested in interactions $I$ that satisfy \emph{differential privacy}.
\begin{definition}[\cite{DMNS06}]
Two datasets $S, S' \in \cX^n$ are neighbors if they differ in at most one coordinate. An interaction $\textrm{Interact}(M, \cdot \,; \cdot)$ satisfies $(\epsilon,\delta)$-differential privacy if for all data analysts $\cA$, pairs of neighboring datasets $S, S' \in \cX^n$, and for all events $E \subseteq \pispace$:
$$ \Pr[ \textrm{Interact}(M, \cA; S)\in E] \leq e^\epsilon \cdot \Pr[ \textrm{Interact}(M, \cA; S') \in E]+\delta$$
where the $\Pr[\cdot]$ operator denotes either a probability density or a probability mass. If  $\textrm{Interact}(M, \cdot \,; \cdot)$ satisfies $(\epsilon,\delta)$-differential privacy, we will also say that $M$ satisfies $(\epsilon,\delta)$-differential privacy.
\end{definition}

We introduce a novel quantity that will be crucial to our argument: it captures the effect of the transcript on the change in the expectation of a query contained in the transcript.
\begin{definition}
An interaction $\textrm{Interact}(M, \cA ; \cdot)$ is called $(\epsilon,\delta)$-posterior sensitive if for every data distribution $\cP$:
$$\Pr_{S \sim \cP^n, \pirv \sim \textrm{Interact}(M, \cA; S)} [\max_j |q_j (\cP^n) - q_j (\cQ_{\pirv})| \geq \epsilon] \leq \delta.$$
\end{definition}

\section{An Elementary Proof of the Transfer Theorem}
\subsection{A General Transfer Theorem}
\label{sec:general}
In this section we prove a general transfer theorem for sample accurate mechanisms with low posterior sensitivity. In Section \ref{sec:transfer_for_DP} we prove that differentially private mechanisms have low posterior sensitivity.
\begin{theorem}[General Transfer Theorem]
\label{thm:gen_transfer}
Suppose that $\textrm{Interact}\,(M,\cA; \cdot)$ is an $(\alpha,\beta)$-sample accurate, $(\epsilon, \delta)$-posterior sensitive interaction. Then for every $c > 0$ it also satisfies:
$$\Pr_{S \sim \cP^n, \, \pirv \sim \textrm{Interact}\,(M,\cA; S)}[\max_{j} |a_j - q_j(\cP)| > \alpha +c + \epsilon] \leq \frac{\beta}{c} + \delta$$
i.e. it is $(\alpha', \beta')$-distributionally accurate for $\alpha' = \alpha +c + \epsilon $ and $\beta' = \frac{\beta}{c} + \delta$.
\end{theorem}

The theorem follows easily from a change in perspective driven by an elementary observation. Imagine that \emph{after} the interaction is run and results in a transcript $\pi$, the dataset $S$ is \emph{resampled} from its posterior distribution $\cQ_\pi$. This does not change the joint distribution on datasets and transcripts. This simple claim is formalized below: its elementary proof appears in Appendix \ref{sec:app1}.

\begin{lemma}[Bayesian Resampling Lemma]
\label{lem:resamp}
Let $E \subseteq \cX^n \times \pispace$ be any event. Then:
$$\Pr_{S \sim \cP^n, \pirv \sim I(S)}[(S, \pirv) \in E] = \Pr_{S \sim \cP^n, \pirv \sim I(S), S' \sim \cQ_\pirv}[(S', \pirv) \in E]$$
\end{lemma}

The change in perspective suggested by the resampling lemma makes it easy to see why the following must be true: any sample-accurate mechanism must in fact be accurate with respect to the posterior distribution it induces. This is because if it can first commit to answers, and guarantee that they are sample-accurate \emph{after} the dataset is resampled from the posterior, the answers it committed to must have been close to the posterior means, because it is likely that the empirical answers on the resampled dataset will be. This argument is generic and does not use differential privacy.

\begin{lemma}
\label{lem:posterioraccuracy1}
Suppose that $M$ is $(\alpha,\beta)$-sample accurate. Then for every $c > 0$ it also satisfies:
$$\Pr_{S \sim \cP^n, \, \pirv \sim \textrm{Interact}\,(M,\cA; S)}[\max_{j} |a_j - q_j( \cQ_\pirv)| > \alpha +c] \leq \frac{\beta}{c}$$
\end{lemma}
\begin{proof}
Denote by $j^*(\pi) = \underset{j}{\argmax} |a_j - q_j( \cQ_\pi)|$. Given $\alpha \geq 0$ and $c > 0$, and expanding the definition of $q_{j^{*}(\pirv)}(\cQ_{\pirv})$  we get:
\begin{align*}
\underset{S \sim \cP^{n}, \pirv \sim I(S)}{\text{Pr}}& \left[a_{j^{*}(\pirv)} - q_{j^{*}(\pirv)}(\cQ_{\pirv}) > \alpha + c \right]\\
= & \underset{S \sim \cP^{n}, \pirv \sim I(S)}{\text{Pr}} \left[\underset{S' \sim \cQ_{\pirv}}{\mathbb{E}}\left[a_{j^{*}(\pirv)} - q_{j^{*}(\pirv)}(S') - \alpha \right] > c\right] \\
\leq & \underset{S \sim \cP^{n}, \pirv \sim I(S)}{\text{Pr}} \left[ \underset{S' \sim \cQ_{\pirv}}{\mathbb{E}} \left[\max \left\{ a_{j^{*}(\pirv)} - q_{j^{*}(\pirv)}(S') - \alpha, 0 \right\} \right] > c \right]\\
\overset{(1)}{\leq} & \frac{1}{c} \underset{S \sim \cP^{n}, \pirv \sim I(S)}{\mathbb{E}} \left[ \underset{S' \sim \cQ_{\pirv}}{\mathbb{E}} \left[\max \left\{ a_{j^{*}(\pirv)} - q_{j^{*}(\pirv)}(S') - \alpha, 0 \right\} \right] \right]\\
\overset{(2)}{\leq} & \frac{1}{c} \underset{S \sim \cP^{n}, \pirv \sim I(S)}{\mathbb{E}} \left[ \underset{S' \sim \cQ_{\pirv}}{\text{Pr}} \left[a_{j^{*}(\pirv)} - q_{j^{*}(\pirv)}(S') - \alpha > 0 \right] \right] \\
= & \frac{1}{c} \underset{S \sim \cP^{n}, \pirv \sim I(S), S' \sim \cQ_{\pirv}}{\text{Pr}} \left[a_{j^{*}(\pirv)} - q_{j^{*}(\pirv)}(S') > \alpha \right] \\
\overset{(3)}{=} & \frac{1}{c} \underset{S \sim \cP^{n}, \pirv \sim I(S)}{\text{Pr}} \left[a_{j^{*}(\pirv)} - q_{j^{*}(\pirv)}(S) > \alpha \right]
\end{align*}
Here, inequality (1) follows from Markov's inequality, inequality (2) follows from the fact that $a_{j^{*}(\pirv)} - q_{j^{*}(\pirv)}(S') - \alpha \leq 1$, and equality 3 follows from the Bayesian Resampling Lemma (Lemma \ref{lem:resamp}). 
Repeating this argument for $q_{j^{*}(\pirv)}(Q_{\pirv}) - a_{j^{*}(\pirv)}$ yields a symmetric bound, so by combining the two with the guarantee of $(\alpha,\beta)$-sample accuracy we get,
$$
\underset{S \sim \cP^{n}, \pirv \sim I(S)}{\text{Pr}} \left[ \left| a_{j^{*}(\pirv)} - q_{j^{*}(\pirv)}(Q_{\pirv}) \right| > \alpha + c \right] \leq \frac{1}{c} \underset{S \sim \cP^{n}, \pirv \sim I(S)}{\text{Pr}} \left[ \left| a_{j^{*}(\pirv)} - q_{j^{*}(\pirv)}(S) \right| > \alpha \right] \leq \frac{\beta}{c}
$$
\end{proof}

Because sample accuracy implies accuracy with respect to the posterior distribution, together with a bound on posterior sensitivity, the transfer theorem follows immediately:
\begin{proof}[Proof of Theorem \ref{thm:gen_transfer}]
By the triangle inequality:
$$\max_j |a_{j} - q_{j}(\cP)|  \leq \max_{i} |a_{i} - q_{i}( \cQ_\pirv)| + \max_l |q_{l}( \cQ_\pirv) - q_{l}(\cP)|.$$

Lemma \ref{lem:posterioraccuracy1} bounds the first term by $\alpha + c$ with probability $1-\frac{\beta}{c}$ over $\pirv$, and the definition of posterior sensitivity bounds the second term by $\epsilon$ with probability $1-\delta$ over $\pirv$, which concludes the proof.
\end{proof} 

\subsection{A Transfer Theorem for Differential Privacy}
\label{sec:transfer_for_DP}
In this section we  prove a transfer theorem for differentially private mechanisms by demonstrating that they have low posterior sensitivity and applying our general transfer theorem.

We here show that differentially private mechanisms are posterior-sensitive for linear queries. In the Appendix we extend this argument to low-sensitivity and optimization queries.
\begin{lemma}
\label{lem:expdelta}
If $M$ is $(\epsilon,\delta)$-differentially private, then for any data distribution $\cP$, any analyst $\cA$, and any constant $c > 0$:
$$
\Pr_{S \sim \cP^n, \pirv \sim \textrm{Interact}\,(M,\cA; S)} \left[ \max_j \left| q_j (\cQ_\pirv)  - q_j (\cP) \right| > (e^\epsilon - 1) + 2c \right] \le \frac{\delta}{c}
$$
i.e. it is $(\epsilon', \delta')$-posterior sensitive for every data analyst $\cA$, where $\epsilon' = e^{\epsilon}-1 + 2 c$ and $\delta' = \frac{\delta}{c}$.
\end{lemma}
\begin{proof}
Given a transcript $\pi \in \pispace$, let $j^*(\pi) \in \argmax_{j} \left| q_j (\cQ_\pi)  - q_j (\cP) \right|$. define for an $\alpha > 0$:
$$
\pispace_\alpha = \left\{ \pi \in \pispace | \, q_{j^*(\pi)} (\cQ_\pi) - q_{j^*(\pi)} (\cP) > \alpha \right\}
$$
$$
\cX^{+} (\pi) = \left\{ x \in \cX | \, \Pr_{S \sim \cQ_\pi, S_i \sim S}  \left[ S_i = x \right] > \Pr_{S_i \sim \cP} \left[ S_i = x \right] \right\}
$$
$$
B_\alpha^+ = \bigcup_{\pi \in \pispace_\alpha} \left( \cX^{+} (\pi) \times \{ \pi \} \right)
$$
$$
\pispace_\alpha^+(x) = \left\{ \pi \in \pispace | \, (x,\pi) \in B_\alpha^+ \right\}
$$
Fix any $\alpha$. Suppose that $\Pr \left[ \left| q_{j^*(\pirv)} (\cQ_\pirv)  - q_{j^*(\pirv)} (\cP) \right| > \alpha \right] > \frac{\delta}{c}$. We must have that either \\ $\Pr \left[ q_{j^*(\pirv)} (\cQ_\pirv)  - q_{j^*(\pirv)} (\cP) > \alpha \right] > \frac{\delta}{2c}$ or $\Pr \left[ q_{j^*(\pirv)} (\cP) - q_{j^*(\pirv)} (\cQ_\pirv)  > \alpha \right] > \frac{\delta}{2c}$. Without loss of generality, assume
\begin{equation}\label{eq:delta1}
\Pr \left[ q_{j^*(\pirv)} (\cQ_\pirv)  - q_{j^*(\pirv)} (\cP) > \alpha \right] = \Pr \left[\pirv \in \pispace_\alpha \right] > \frac{\delta}{2c}
\end{equation}
Let $S_i$ be the random variable obtained by first sampling $S \sim \cP^n$ and then sampling $S_i \in S$ uniformly at random. We compare the probability measure of $B_\alpha^+$ under the joint distribution on $S_i$ and $\pirv$ with its corresponding measure under the product distribution of $S_i$ and $\pirv$:
\begin{align*}
&\Pr_{(S_i, \pirv)} \left[ (S_i, \pirv) \in B_\alpha^+ \right] - \Pr_{S_i \otimes \pirv } \left[ (S_i, \pirv) \in B_\alpha^+ \right] \\
&= \sum_{\pi \in \pispace_\alpha} \Pr [\pirv = \pi ] \sum_{x \in \cX^{+} (\pi)} \left( \Pr [S_i = x | \pirv = \pi] - \Pr [S_i = x] \right) \\
&\ge \sum_{\pi \in \pispace_\alpha} \Pr [\pirv = \pi] \sum_{x \in \cX^{+} (\pi) } q_{j^*(\pi)} (x) \left( \Pr [S_i = x | \pirv = \pi] - \Pr [S_i = x] \right) \\
&\ge \sum_{\pi \in \pispace_\alpha} \Pr [\pirv = \pi] \sum_{x \in \cX } q_{j^*(\pi)} (x) \left( \Pr [S_i = x | \pirv = \pi] - \Pr [S_i = x] \right) \\
&= \sum_{\pi \in \pispace_\alpha} \Pr [\pirv = \pi] \left( q_{j^*(\pi)} (\cQ_\pi) - q_{j^*(\pi)} (\cP) \right) \\
&> \alpha \cdot \Pr \left[ \pirv \in \pispace_\alpha \right]
\end{align*}
On the other hand, using the definition of $(\epsilon, \delta)$-differential privacy (See Lemma \ref{lem:tech2} for the elementary derivation of the first inequality):
\begin{align*}
&\Pr_{(S_i, \pirv)} \left[ (S_i, \pirv) \in B_\alpha^+ \right] - \Pr_{S_i \otimes \pirv } \left[ (S_i, \pirv) \in B_\alpha^+ \right] \\
&= \sum_{x \in \cX} \Pr [ S_i = x ] \left( \Pr \left[\pirv \in \pispace_\alpha^+(x) | S_i = x\right] - \Pr \left[\pirv \in \pispace_\alpha^+(x) \right] \right) \\
&\le \sum_{x \in \cX} \Pr [ S_i = x ] \left( (e^\epsilon - 1) \Pr \left[\pirv \in \pispace_\alpha^+(x) \right] + \delta \right) \\
&= (e^\epsilon - 1) \Pr_{S_i \otimes \pirv } \left[ (S_i, \pirv) \in B_\alpha^+ \right] + \delta \\
&\le (e^\epsilon - 1) \Pr \left[\pirv \in \pispace_\alpha \right] + \delta \\
& < (e^\epsilon - 1) \Pr \left[\pirv \in \pispace_\alpha \right] + 2c \Pr \left[\pirv \in \pispace_\alpha \right] \quad \text{(by Equation (\ref{eq:delta1}))}\\
&= ((e^\epsilon - 1) + 2c) \cdot \Pr \left[\pirv \in \pispace_\alpha \right]
\end{align*}
This is a contradiction for $\alpha \geq (e^\epsilon - 1) + 2c$.
\end{proof}

\begin{remark}
Note
\begin{enumerate}
\item Since differential privacy is closed under post processing, this claim can be generalized beyond queries contained in the transcript to any query generated as function of the transcript.

\item In the case of $(\epsilon, 0)$-differential privacy, choosing $c=0$, the claim holds for every query with probability 1.
\end{enumerate}
\end{remark}

Combined with our general transfer theorem (Theorem \ref{thm:gen_transfer}), this directly yields a transfer theorem for differential privacy:

\begin{theorem}[Transfer Theorem for $(\epsilon,\delta)$-Differential Privacy]
\label{thm:transfer2}
Suppose that $M$ is $(\epsilon,\delta)$-differentially private and $(\alpha,\beta)$-sample accurate for linear queries. Then for every analyst $\cA$ and $c,d > 0$ it also satisfies:
$$\Pr_{S \sim \cP^n, \, \pirv \sim \textrm{Interact}\,(M,\cA; S)}[\max_{j} |a_j - q_j(\cP)| > \alpha +(e^\epsilon-1)+ c + 2d] \leq \frac{\beta}{c} + \frac{\delta}{d}$$
i.e. it is $(\alpha', \beta')$-distributionally accurate for $\alpha' = \alpha +(e^\epsilon-1)+ c + 2d$ and $\beta' = \frac{\beta}{c} + \frac{\delta}{d}$.
\end{theorem}

\begin{remark}
As we will see in Section \ref{sec:app}, the Gaussian mechanism (and many other differentially private mechanisms) has a sample accuracy bound that depends only on the square root of the log of both $1/\beta$ and $1/\delta$. Thus, despite the Markov-like term $\beta' = \frac{\beta}{c} + \frac{\delta}{d}$ in the above transfer theorem, together with the sample accuracy bounds of the Gaussian mechanism, it yields Chernoff-like concentration.
\end{remark}

Our technique extends easily to reason about arbitrary low sensitivity queries and minimization queries. See Appendix \ref{sec:lowsens} and \ref{sec:optquery} for more details.

\section{Applications: The Gaussian Mechanism}
\label{sec:app}

We now apply our new transfer theorem to derive the concrete bounds that we plotted in Figure \ref{fig:intro}.
The Gaussian mechanism is extremely simple and has only a single parameter $\sigma$: for each query $q_i$ that arrives, the Gaussian mechanism returns the answer $a_i \sim \mathcal{N}(q_i(S), \sigma^2)$ where $\mathcal{N}(q_i(S),\sigma^2)$ denotes the Gaussian distribution with mean $q_i(S)$ and standard deviation $\sigma$. First, we recall the differential privacy properties of the Gaussian mechanism.
\begin{theorem}[\cite{BS16}]
\label{thm:gaussprivacy}
When used to answer $k$ linear queries, the Gaussian mechanism with parameter $\sigma$ satisfies $\rho$-zCDP for $\rho = \frac{k}{2n^2\sigma^2}$. A consequence of this is that for every $0 < \delta < 1$, it satisfies $(\epsilon,\delta)$-differential privacy for:
$$\epsilon = \frac{k}{2n^2\sigma^2}+\sqrt{2\frac{k}{n^2\sigma^2}\log\left(\sqrt{\pi\cdot\frac{k}{2n^2\sigma^2}}/\delta\right)}$$
\end{theorem}
It is also easy to see that the sample-accuracy of the Gaussian mechanism is characterized by the CDF of the Gaussian distribution:
\begin{lemma}
\label{thm:gaussaccuracy}
For any $0 < \beta < 1$, the Gaussian mechanism with parameter $\sigma$ is $(\alpha_G,\beta)$-sample accurate for:
$$\alpha_G = \sqrt{2}\sigma\cdot  \mathrm{erfc}^{-1}\left(2-2\left(1-\frac{\beta}{2}\right)^{1/k}\right) <\sqrt{2}\sigma \cdot \mathrm{erfc}^{-1}\left(\frac{\beta}{k}\right)<\sqrt{2}\sigma \sqrt{\log\left(\frac{\sqrt{2}k}{\pi \beta}\right)}.$$
Above, $\mathrm{erfc}(x) = 1-\mathrm{erf}(x)$ is the complementary error function.
\end{lemma}
\begin{proof}
For a query $q_j$, write $a_j = q_j(S) + Z_j$ where $Z_j \sim \mathcal{N}(0,\sigma^2)$. The sample error is $\max_j |a_j - q_j(S)| = \max_j |Z_j|$. We have that $\Pr[\max_j |Z_j| \geq \alpha] \le \Pr[\max_j Z_j \geq \alpha] + \Pr[\min_j Z_j \leq -\alpha]$. $\alpha_G$ is the value that solves the equation $\Pr[\max_j Z_j \geq \alpha] = \Pr[\min_j Z_j \leq -\alpha] = \beta/2$
\end{proof}

With these quantities in hand, we can now apply Theorem \ref{thm:transfer2} to derive distributional accuracy bounds for the Gaussian mechanism:
\begin{theorem}
Fix a desired confidence parameter $0 < \beta < 1$. When $\sigma$ is set optimally, the Gaussian mechanism can be used to answer $k$ linear queries while satisfying $(\alpha,\beta)$-distributional accuracy, where $\alpha$ is the solution to the following unconstrained minimization problem:
$$\alpha = \min_{\sigma,\delta > 0}  \left\{ \sqrt{2}\sigma \cdot \mathrm{erfc}^{-1}\left(\frac{\delta}{k}\right)+e^{ \frac{k}{2n^2\sigma^2}+\sqrt{2\frac{k}{n^2\sigma^2}\log\left(\sqrt{\pi\cdot\frac{k}{2n^2\sigma^2}}/\delta\right)}}-1+6\left(\frac{\delta}{\beta}\right) \right\} $$
\begin{proof}
Using Theorem \ref{thm:transfer2} and fixing $\beta' = \delta$ and $c = d$, we have that an $(\alpha',\beta')$-sample accurate, $(\epsilon,\delta)$-differentially private mechanism is $(\alpha,\beta)$-distributionally accurate for $\alpha = \alpha' +(e^\epsilon-1)+3c$ and $\beta = \frac{2\delta}{c}$ where $c$ can be an arbitrary parameter. For any fixed value of $\beta$, we can take $c = \frac{2\delta}{\beta}$, and see that we obtain $(\alpha,\beta)$-distributional accuracy where $\alpha =  \alpha' +(e^\epsilon-1)+6\left(\delta / \beta\right)$. The theorem then follows from plugging in the privacy bound from Theorem \ref{thm:gaussprivacy}, the sample accuracy bound from Theorem \ref{thm:gaussaccuracy}, and optimizing over the free variables $\sigma$ and $\delta$.
\end{proof}

\end{theorem} 

\section{Discussion}
We have given a new proof of the transfer theorem for differential privacy that has several appealing properties. Besides being simpler than previous arguments, it achieves substantially better concrete bounds than previous transfer theorems, and uncovers new structural insights about the role of differential privacy and sample accuracy. In particular, sample accuracy serves to guarantee that the reported answers are close to their posterior means, and differential privacy serves to guarantee that the posterior means are close to their true answers. This focuses attention on the posterior data distribution as a key quantity of interest, which we expect will be fruitful in future work. In particular, it may shed light on what makes certain data analysts overfit less than worst-case bounds would suggest: because they choose queries whose posterior means are closer to the prior than the worst-case query.

There seems to be one remaining place to look for improvement in our transfer theorem:  Lemmas \ref{lem:posterioraccuracy1} and \ref{lem:expdelta} both exhibit a Markov-like tradeoff between a parameter $c$ and $\beta$ and $\delta$ respectively. Although the dependence on $\beta$ and $\delta$ in our ultimate bounds is only root-logarithmic, it would still yield an improvement if this Markov-like dependence could be replaced with a Chernoff-like dependence. It \emph{is} possible to do this for the $\beta$ parameter: we give an alternative (and even simpler) proof of the transfer theorem for $(\epsilon,0)$-differential privacy which shows that posterior distributions induced by private mechanisms exhibit Chernoff-like concentration, in Appendix \ref{app:better}. But the only way we know to extend this argument to $(\epsilon,\delta)$-differential privacy requires dividing $\delta$ by a factor of $n$, which yields a final theorem that is inferior to Theorem \ref{thm:transfer2}. 

\paragraph{Acknowledgements} We thank Adam Smith for helpful conversations at an early stage of this work.

\bibliographystyle{plainnat}
\bibliography{ref}

\appendix
\section{Extensions}
\subsection{Low Sensitivity Queries}
\label{sec:lowsens}

Our technique extends easily to reason about arbitrary \emph{low sensitivity} queries. We only need to generalize our lemma about posterior sensitivity.
\begin{definition}
A query $q:\cX^n\rightarrow \mathbb{R}$ is called $\Delta$-sensitive if for all pairs of neighbouring datasets $S, S' \in \cX^n$: $|q(S)-q(S')| \leq \Delta$. Note that linear queries are $(1/n)$-sensitive.
\end{definition}

\begin{lemma}\label{lem:lowsens}
If $M$ is an $(\epsilon,\delta)$-differentially private mechanism for answering $\Delta$-sensitive queries, then for any data distribution $\cP$, analyst $\cA$, and any constant $c > 0$:
$$
\Pr_{S \sim \cP^n, \pirv \sim \textrm{Interact}\,(M,\cA; S))} \left[\max_{j} \left| q_j (\cQ_\pirv)  - q_j (\cP^n) \right| > (e^\epsilon - 1 + 4c) n \Delta \right] \le \frac{\delta}{c}
$$
i.e. it is $(\epsilon',\frac{\delta}{c})$-posterior sensitive for every $\cA$, where $\epsilon' =  (e^\epsilon - 1 + 4c) n \Delta$.
\end{lemma}

\begin{proof}
We introduce a useful bit of notation: $\bar{q}\left(\boldx_{\le i} \right) = \underset{S' \sim \cP^{n-i}}{\mathbb{E}} \left[ q\left( \left(\boldx_{\le i}, S' \right) \right) \right]$. Notice that $\bar{q}\left(\boldx_{\le 0} \right) = q \left(\cP^n \right)$ and $\bar{q}\left(\boldx_{\le n} \right) = q \left(\boldx \right)$. Given a transcript $\pi \in \pispace$, let $j^*(\pi) \in \argmax_{j} \left| q_j (\cQ_\pi)  - q_j (\cP^n) \right|$. Denote for any $\alpha \ge 0$
$$\pispace_{\alpha} = \left\{ \pi \in \pispace \,|\, q_{j^*(\pi)} \left(\cQ_\pi \right) - q_{j^*(\pi)} \left(\cP^n \right) > \alpha \right\}$$
and for any $z\in \left[0, 2\Delta \right]$ denote
$$
\pispace_{\alpha, z}\left(\boldx_{\le i} \right) = \left\{ \pi \in \pispace_{\alpha} \,|\, \bar{q}_{j^*(\pi)} \left(\boldx_{\le i} \right) - \bar{q}_{j^*(\pi)} \left(\boldx_{\le i-1} \right) > z - \Delta \right\}
$$
From the definition of differential privacy:
\begin{eqnarray*}
& &\E_{S \sim \cP^n} \left[\sum_{\pi \in \pispace_{\alpha}} \Pr_{\pirv \sim I(S)} \left[\pirv = \pi \right]  \left( \bar{q}_{j^*(\pi)} \left(S_{\le i} \right) - \bar{q}_{j^*(\pi)} \left(S_{\le i-1} \right) + \Delta \right) \right]\\
&= & \E_{S \sim \cP^n} \left[\int_{0}^{2\Delta} \Pr_{\pirv \sim I(S)} \left[\pirv \in \pispace_{\alpha, z}\left( S_{\le i} \right) \right] dz\right] \\
&\le & \E_{S \sim \cP^n, Y \sim \cP} \left[\int_{0}^{2\Delta} \left( e^{\epsilon} \Pr_{\pirv \sim I(S^{i \leftarrow Y})} \left[\pirv \in \pispace_{\alpha, z}\left(S_{\le i} \right) \right] + \delta \right) dz\right] \\
&= &\E_{S \sim \cP^n, Y \sim \cP} \left[ e^{\epsilon} \sum_{\pi \in \pispace_{\alpha}}\Pr_{\pirv \sim I(S^{i \leftarrow Y})} \left[\pirv = \pi\right] \left( \bar{q}_{j^*(\pi)} \left(S_{\le i} \right) - \bar{q}_{j^*(\pi)} \left(S_{\le i-1}\right) + \Delta \right) + 2 \Delta \delta\right] \\
&= & \E_{S \sim \cP^n, Y \sim \cP} \left[ e^{\epsilon} \sum_{\pi \in \pispace_{\alpha}}\Pr_{\pirv \sim I(S)} \left[\pirv = \pi\right] \left( \bar{q}_{j^*(\pi)} \left(S_{\le i}^{i \leftarrow Y} \right) - \bar{q}_{j^*(\pi)} \left(S_{\le i-1}\right) + \Delta \right) + 2 \Delta \delta\right] \\
\end{eqnarray*}
where $S^{i \leftarrow Y} = (S_1, \ldots, S_{i-1}, Y, S_{i+1}, \ldots, S_n)$, and the last equality follows from the observation that $(S, Y)$ and $(S^{i \leftarrow Y}, S_{i)}$ are identically distributed. Since $Y \sim \cP$, independently from $\pirv$, we get that $\mathbb{E}_{Y \sim \cP} \left[ \bar{q}_{j^*(\pi)} \left(S_{\le i}^{i \leftarrow Y} \right) \right] = \bar{q}_{j^*(\pi)} \left(S_{\le i-1}\right)$, so
\begin{align*}
 \E_{S \sim \cP^n} \left[ \sum_{\pi \in \pispace_{\alpha}} \Pr_{\pirv \sim I(S)} \left[\pirv = \pi \right] \left( \bar{q}_{j^*(\pi)} \left(S_{\le i} \right) - \bar{q}_{j^*(\pi)} \left(S_{\le i-1} \right) + \Delta \right) \right] &\leq \E_{S \sim \cP^n} \left[\left( e^{\epsilon} \Pr_{\pirv \sim I(S)} \left[\pirv \in \pispace_{\alpha}\right] + 2 \delta \right) \Delta\right] \\
&= \left( e^{\epsilon} \Pr \left[\pirv \in \pispace_{\alpha}\right] + 2 \delta \right) \Delta
\end{align*}
Subtracting $\Delta \Pr \left[\pirv \in \pirv_{\alpha}\right]$ from both sides we get
\begin{equation}\label{eq:sth}
\underset{S \sim \cP^{n}}{\mathbb{E}} \left[ \sum_{\pi \in \pispace_{\alpha}} \Pr_{\pirv \sim I(S)} \left[\pirv = \pi \right] \left( \bar{q}_{j^*(\pi)} \left(S_{\le i} \right) - \bar{q}_{j^*(\pi)} \left(S_{\le i-1} \right) \right) \right] \le \left( \left(e^{\epsilon} - 1 \right) \Pr \left[\pirv \in \pispace_{\alpha}\right] + 2 \delta \right) \Delta
\end{equation}

We now chooose $\alpha = \left( e^{\epsilon} - 1 + 4 c \right) n \Delta$. Suppose that $\Pr \left[ \left| q_{j^*(\pirv)}  (\cQ_\pirv)  - q_{j^*(\pirv)} (\cP^n) \right| > \alpha \right] > \frac{\delta}{c}$. We must have that either $\Pr \left[ q_{j^*(\pirv)} (\cQ_\pirv)  - q_{j^*(\pirv)} (\cP^n) > \alpha \right] > \frac{\delta}{2c}$ or $\Pr \left[ q_{j^*(\pirv)} (\cP^n) - q_{j^*(\pirv)} (\cQ_\pirv)  > \alpha \right] > \frac{\delta}{2c}$.
Without loss of generality, assume
\begin{equation}\label{eq:sth2}
\Pr \left[ q_{j^*(\pirv)} (\cQ_\pirv)  - q_{j^*(\pirv)} (\cP^n) > \alpha \right] = \Pr \left[\pirv \in \pispace_\alpha \right] > \frac{\delta}{2c}
\end{equation}
But this leads to a contradiction, since
\begin{align*}
\Pr \left[\pirv \in \pispace_{\alpha}\right] \left( e^{\epsilon} - 1 + 4 c \right) n \Delta < & \sum_{\pi \in \pispace_{\alpha}} \Pr \left[\pirv = \pi \right] \left( q_{j^*(\pi)} (\cQ_\pi)  - q_{j^*(\pi)} (\cP^n) \right) \\
= & \underset{S \sim \cP^{n}}{\mathbb{E}} \left[ \sum_{\pi \in \pispace_{\alpha}} \Pr_{\pirv \sim I(S)} \left[\pirv = \pi \right] \left( q_{j^*(\pi)} \left(S \right) - q_{j^*(\pi)} \left(\cP^n \right) \right) \right] \\
= & \sum_{i=1}^{n} \underset{S \sim \cP^{n}}{\mathbb{E}} \left[ \sum_{\pi \in \pispace_{\alpha}} \Pr_{\pirv \sim I(S)} \left[\pirv = \pi \right] \left( \bar{q}_{j^*(\pi)} \left(S_{\le i} \right) - \bar{q}_{j^*(\pi)} \left(S_{\le i - 1} \right) \right) \right] \\
\le & \left( \left(e^{\epsilon} - 1 \right) \Pr \left[\pirv \in \pispace_{\alpha}\right] + 2 \delta \right) n \Delta \quad \text{(by Equation (\ref{eq:sth}))} \\
< & \Pr \left[\pirv \in \pispace_{\alpha}\right] \left( e^{\epsilon} - 1 + 4 c \right) n \Delta \quad \text{(by Equation (\ref{eq:sth2}))}
\end{align*}

\end{proof}

We can combine this Lemma with Lemma \ref{lem:posterioraccuracy1} (which holds for any query type) to get our transfer theorem:
\begin{theorem}[Transfer Theorem for Low Sensitivity Queries]
\label{thm:transfer-lowsens}
Suppose that $M$ is $(\epsilon,\delta)$-differentially private and $(\alpha,\beta)$-sample accurate for $\Delta$-sensitive queries. Then for every analyst $\cA$, $c,d > 0$ it also satisfies:
$$\Pr_{S \sim \cP^n, \, \pirv \sim \textrm{Interact}\,(M,\cA;S)} \left[ \max_j \left|a_j - q_j(\cP^n) \right| > \alpha +c+ (e^\epsilon - 1 + 4d + c) n \Delta \right] \leq \frac{\beta}{c} + \frac{\delta}{d}$$
i.e. it is $(\alpha', \beta')$-distributionally accurate for $\alpha' = \alpha +c+ (e^\epsilon - 1 + 4d + c) n \Delta$ and $\beta' =  \frac{\beta}{c} + \frac{\delta}{d}$.
\end{theorem} 
\subsection{Minimization Queries}
\label{sec:optquery}
\begin{definition}
Minimization queries are specified by a loss function $L:\cX^n \times \Theta \rightarrow [0,1]$ where $\Theta$ is generally known as the ``parameter space". An answer to a minimization query $L$ is a parameter $\theta \in \Theta$. We work with $\Delta$-sensitive minimization queries: for all pairs of neighbouring datasets $S, S' \in \cX^n$ and all $\theta \in \Theta$, $| L(S,\theta)-L(S', \theta)| \leq \Delta$.

A mechanism $M$ is $(\alpha, \beta)$-sample accurate for minimization queries if for every data analyst $\cA$ and every dataset $S \in \cX^n$:
$$\Pr_{\pirv \sim \textrm{Interact}(M, \cA; S)} \left[ \max_j\left| L_j(S, \theta_j) - \min_{\theta \in \Theta} L_j(S, \theta) \right| \geq \alpha \right] \leq \beta$$
We say that $M$ satisfies $(\alpha,\beta)$-distributional accuracy  for  minimization queries if for every data analyst $\cA$ and every data distribution $\cP$:
$$\Pr_{S \sim \cP^n, \pirv \sim \textrm{Interact}(M, A; S)} \left[ \max_j \left|   \E_{S' \sim \cP^n} \left[ L_j(S', \theta_j) - \min_{\theta \in \Theta} L_j(S', \theta) \right] \right| \geq \alpha \right] \leq \beta$$
\end{definition}

\begin{remark}
Note that
$$
 \E_{S' \sim \cP^n} \left[ L_j(S', \theta_j) \right] - \min_{\theta \in \Theta} \E_{S' \sim \cP^n} \left[ L_j(S', \theta) \right] \le \E_{S' \sim \cP^n} \left[ L_j(S', \theta_j) - \min_{\theta \in \Theta} L_j(S', \theta) \right]
$$
So as long as the RHS is bounded, the LHS is bounded too.
\end{remark}

\begin{remark}
For a given $\Delta$-sensitive minimization query $L_j$ and an answer $\theta_j$, define:
$$q_j (S) := L_j(S,\theta_j) - \min_{\theta \in \Theta} L_j (S, \theta) \quad \textrm{and} \quad a_j  := 0$$
Note several things:
\begin{enumerate}
\item If $L_j$ is $\Delta$-sensitive, then $q_j$ is 2$\Delta$-sensitive.
\item The mapping from a minimization query transcript $\pi = ((L_1,\theta_1),\ldots,(L_k,\theta_k))$ to the $2\Delta$-sensitive query transcript $\pi' = ((q_1,a_1),\ldots,(q_k,a_k))$ as defined above is a dataset-independent post-processing $\pi' = f(\pi)$.
\item $\pi$ satisfies an $(\alpha,\beta)$-accuracy guarantee if and only if $\pi'$ does.
\end{enumerate}
\end{remark}
With the above observation, the transfer theorem for minimization queries immediately follows by Lemma \ref{lem:lowsens} and Lemma \ref{lem:posterioraccuracy1}.

\begin{theorem}[Transfer Theorem for Minimization Queries]
\label{thm:transfer-lowsens}
Suppose that $M$ is $(\epsilon,\delta)$-differentially private and $(\alpha,\beta)$-sample accurate for $\Delta$-sensitive minimization queries. Then for every analyst $\cA$ and $c,d > 0$ it also satisfies:
$$\Pr_{S \sim \cP^n, \, \pirv \sim \textrm{Interact}\,(M,\cA;S)} \left[\max_j  \left| \E_{S' \sim \cP^n} \left[ L_j(S', \theta_j) -  \min_{\theta \in \Theta} L_j(S', \theta) \right] \right| >\alpha +c+ 2(e^\epsilon - 1 + 4d + c) n \Delta\right] \leq \frac{\beta}{c} + \frac{\delta}{d}$$
i.e. it is $(\alpha', \beta')$-distributionally accurate for $\alpha' = \alpha +c+ 2(e^\epsilon - 1 + 4d + c) n \Delta$ and $\beta' = \frac{\beta}{c} + \frac{\delta}{d}$.
\end{theorem} 
\section{Details from Section \ref{sec:general}}
\label{sec:app1}

\begin{proof}[Proof of Lemma \ref{lem:resamp}]
This follows from the expansion of the definition, and an application of Bayes Rule.
\begin{eqnarray*}
\Pr_{S \sim \cP^n, \pirv \sim I(S), S' \sim \cQ_\pirv}[(S', \pirv) \in E] &=& \sum_{\boldx}\sum_{\pi}\sum_{\boldx'} \Pr[S = \boldx ] \Pr[\pirv = \pi | S = \boldx] \Pr_{S' \sim \cQ_\pi}[S' = \boldx' ] \1 [(\boldx', \pi) \in E]\\
&=& \sum_{\pi}\sum_{\boldx'}\Pr[\pirv = \pi] \Pr_{S' \sim \cQ_\pi} [S' = \boldx'] \1 [(\boldx', \pi) \in E]\\
&=& \sum_{\pi}\sum_{\boldx'}\Pr[\pirv = \pi] \Pr[S = \boldx' | \pirv = \pi] \1 [(\boldx', \pi) \in E] \\
&=& \sum_{\pi}\sum_{\boldx'} \Pr[\pirv = \pi] \frac{\Pr[\pirv = \pi | S = \boldx'] \cdot \Pr[S = \boldx']}{\Pr[\pirv = \pi]} \1 [(\boldx', \pi) \in E] \\
&=&\Pr_{S \sim \cP^n, \pirv \sim I(S)}[(S, \pirv) \in E]
\end{eqnarray*}
\end{proof}

\section{Details from Section \ref{sec:transfer_for_DP}}
\begin{lemma}
\label{lem:tech2}
If $M$ is $(\epsilon,\delta)$-differentially private, then for any event $E$ and datapoint $x$:
$$\Pr_{S \sim \cP^n, S_i \sim S, \pirv \sim I(S)}[\pirv \in E | S_i = x] \leq e^\epsilon \Pr_{S \sim \cP^n, \pirv \sim I(S)}[\pirv \in E]+\delta$$
\end{lemma}
\begin{proof}
This follows from expanding the definitions.
\begin{eqnarray*}
\Pr_{S \sim \cP^n, S_i \sim S, \pirv \sim I(S)}[\pirv \in E | S_i = x] &=& \frac{1}{n}\sum_{i=1}^n \Pr_{S \sim \cP^n, \pirv \sim I(S)}[\pirv \in E | S_i = x] \\
&=& \frac{1}{n}\sum_{i=1}^n \sum_{\boldx \in \cX^n} \Pr_{S \sim \cP^n}[S = \boldx]\cdot \Pr[\pirv \in E | S = (\boldx_{-i},x)] \\
&\leq& \frac{1}{n}\sum_{i=1}^n \sum_{\boldx \in \cX^n} \Pr_{S \sim \cP^n}[S = \boldx]\cdot\left(e^\epsilon \Pr[\pirv \in E | S = \boldx]+\delta\right) \\
&=& e^\epsilon  \Pr_{S \sim \cP^n, \pirv \sim I(S)}[\pirv \in E]+\delta
\end{eqnarray*}
where the inequality follows from the definition of differential privacy.
\end{proof}

\section{An (even) Simpler and Better Proof for $\epsilon$-Differential Privacy}
\label{app:better}
In this section we give an \emph{even simpler} proof of an \emph{even better} transfer theorem for $(\epsilon,0$)-differential privacy. Rather than using Markov's inequality as we did in the proof of Lemma \ref{lem:posterioraccuracy1}, we can directly show that posteriors induced by differentially private mechanisms exhibit Chernoff-like concentration.

\begin{lemma}\label{lem:better-posterior}
If $M$ is $(\epsilon,0)$-differentially private, then for any data distribution $\cP$, any transcript $\pi \in \pispace$, any linear query $q$, and any $\eta > 0$:
$$\Pr_{S \sim \cQ_\pi}\left[  \left|q (S)-q (\cP) \right| \geq (e^{\epsilon}-1) + \sqrt{\frac{2\ln(2/\eta)}{n}}\right] \leq \eta$$
\end{lemma}


\begin{proof}
Define the random variables $V_i = q(S_i) - \E[q(S_i) | S_{< i}]$, and let $X_i = \frac{1}{n}\sum_{j=1}^i V_j$. Then the sequence $0=X_0,X_1,\ldots,X_n$ forms a martingale and $|X_i - X_{i-1}| = \frac{1}{n}|V_i| \leq \frac{1}{n}$. We can therefore apply Azuma's inequality to conclude that:
\begin{equation}
\label{eq:azumawarmup}
\Pr\left[\left|\frac{1}{n}\sum_{i=1}^n q(S_i) - \frac{1}{n} \sum_{i=1}^n  \E[q(S_i) | S_{< i}]\right|\geq t \right] \leq 2\exp\left(\frac{-t^2n}{2}\right)
\end{equation}
Now fix any realization $\boldx$, and consider each term: $\E[q(S_i) | S_{< i} = \boldx_{<i}]$. We have:
\begin{eqnarray*}
\E_{S \sim \cQ_\pi}[q(S_i) | S_{< i} = \boldx_{<i}] &=&  \sum_x q(x)\cdot \Pr_{S \sim \cP^n}[S_i = x | \pirv=\pi, S_{<i} =\boldx_{< i}] \\
&=& \sum_x q(x)\cdot \frac{\Pr_{S \sim \cP^n}[\pirv = \pi | S_i = x, S_{< i} = \boldx_{< i}]\cdot \Pr_{S \sim \cP^n}[S_i = x]}{\Pr[\pirv = \pi | S_{< i} =\boldx_{< i}] } \\
&\leq& e^\epsilon \cdot \sum_x q(x)\cdot \Pr_{S \sim \cP^n}[S_i = x] \\
&=& e^\epsilon q(\cP)
\end{eqnarray*}
where the inequality follows from the definition of $(\epsilon,0)$-differential privacy. Symmetrically, we can show that $\E_{S \sim \cQ_\pi}[q(S_i) | S_{< i} = \boldx_{<i}]\geq e^{-\epsilon} q(\cP)$. Therefore we have that:
 $$e^{-\epsilon}q(\cP) \leq \frac{1}{n} \sum_{i=1}^n  \E[q(S_i) | S_{< i}]\leq e^\epsilon q(\cP).$$ Combining this with Equation \ref{eq:azumawarmup} gives us that for any $\eta > 0$, with probability $1-\eta$ when $S \sim \cQ_\pi$:
 $$q(S) \leq e^\epsilon q(\cP) + \sqrt{\frac{2\ln(2/\eta)}{n}} \ \ \textrm{and} \ \ q(S) \geq e^{-\epsilon} q(\cP) - \sqrt{\frac{2\ln(2/\eta)}{n}}$$
\end{proof}
A transfer theorem follows immediately from lemma \ref{lem:better-posterior}.
\begin{theorem}
Suppose that $M$ is $(\epsilon,0)$-differentially private and $(\alpha,\beta)$-sample accurate. Then for any $\eta > 0$ it is $(\alpha', \beta')$-distributionally accurate for $\alpha' =  \alpha +(e^\epsilon-1) + \sqrt{\frac{2\ln(2/\eta)}{n}}$ and $\beta' = \beta +\eta$.
\end{theorem}
\begin{proof}
For a given $\pi$, let $j^*(\pi) = \argmax_j |a_j - q_j(\cP)|$. By the triangle inequality we have:
$$|a_{j^*(\pirv)} - q_{j^*(\pirv)} (\cP)| \leq |a_{j^*(\pirv)} - q_{j^*(\pirv)} (S)| + |q_{j^*(\pirv)} (S) - q_{j^*(\pirv)} (\cP)| \le \max_j |a_{j} - q_{j} (S)| + |q_{j^*(\pirv)} (S) - q_{j^*(\pirv)} (\cP)| $$
By the definition of $(\alpha,\beta)$-sample accuracy, we have that with probability $1-\beta$, $\max_j |a_j - q_j(S)| \leq \alpha$.
The Bayesian Resampling Lemma (Lemma \ref{lem:resamp}) gives us that:
\begin{align*}
&\Pr_{S \sim \cP^n, \, \pirv \sim I(S)}\left[|q_{j^*(\pirv)} (S) - q_{j^*(\pirv)} (\cP)| \geq (e^\epsilon-1) + \sqrt{\frac{2\ln(2/\eta)}{n}}\right] \\
&= \Pr_{S \sim \cP^n, \, \pirv \sim I(S), \, S' \sim \cQ_\pirv}\left[|q_{j^*(\pirv)} (S') - q_{j^*(\pirv)} (\cP)| \geq (e^\epsilon-1) + \sqrt{\frac{2\ln(2/\eta)}{n}}\right] \\
&= \E_{S \sim \cP^n, \, \pirv \sim I(S)} \left[ \Pr_{S' \sim \cQ_\pirv}\left[|q_{j^*(\pirv)} (S') - q_{j^*(\pirv)} (\cP)| \geq (e^\epsilon-1) + \sqrt{\frac{2\ln(2/\eta)}{n}}\right] \right] \\
&\le \eta
\end{align*}
Because Lemma \ref{lem:better-posterior} guarantees us that for \emph{every} $\pi$,
$$\Pr_{S' \sim \cQ_\pi}\left[|q_{j^*(\pi)} (S') - q_{j^*(\pi)} (\cP)| \geq (e^\epsilon-1) + \sqrt{\frac{2\ln(2/\eta)}{n}}\right] \leq \eta.$$
The theorem then follows from a union bound.
\end{proof}

\end{document}